\documentclass[letterpaper, 10 pt, conference]{ieeeconf}  

\IEEEoverridecommandlockouts                              

\usepackage{amsthm}
\usepackage[utf8]{inputenc}
\usepackage{amsmath,amssymb}
\usepackage{booktabs}
\usepackage[font=small]{subcaption}
\usepackage{numprint}
\usepackage{csvsimple}
\usepackage{comment}
\usepackage{amsfonts}
\usepackage{subcaption}
\usepackage{cite}
\usepackage{multirow}

\newtheorem{theorem}{Theorem} 

\newtheorem{lemma}[theorem]{Lemma}
\newtheorem{assumption}{Assumption}
\usepackage[linesnumbered,ruled,noend]{algorithm2e}
\usepackage{algpseudocode}
\usepackage{graphicx}
\usepackage{xcolor}
\usepackage{eso-pic}

\newcommand{\positivereals}{\mathbb{R}_{>0}}
\newcommand{\safe}{\mathrm{safe}}

\title{\LARGE \bf
Chance-Constrained Motion Planning with Event-Triggered Estimation
}

\author{Anne Theurkauf, Qi Heng Ho, Roland Ilyes, Nisar Ahmed, and Morteza Lahijanian
\thanks{This work was supported by NASA STTR award 80NSSC20C0314.}
\thanks{Authors are with the department of Aerospace Engineering Sciences at the University of Colorado Boulder, CO, USA
        {\tt\small \{\textit{firstname}.\textit{lastname}\}@colorado.edu}}
}

\begin{document}

\AddToShipoutPictureBG*{%
  \AtPageUpperLeft{%
    \hspace{16.5cm}%
    \raisebox{-1.5cm}{%
      \makebox[0pt][r]{To appear in the IEEE Int'l. Conference on Robotics and Automation (ICRA), 2023.}}}}

\maketitle

\begin{abstract}
    We consider the problem of motion and communication planning under uncertainty
    with limited information from a remote sensor network. Because the remote sensors are power and bandwidth limited, we use event-triggered (ET) estimation to manage communication costs. We introduce a fast and efficient sampling-based planner which computes 
    motion plans coupled with ET communication strategies that minimize communication costs, while satisfying constraints on the probability of reaching the goal region and the point-wise probability of collision. We derive a novel method for offline propagation of the expected state distribution, and corresponding bounds on this distribution. These bounds are used to evaluate the chance constraints in the algorithm. Case studies establish the validity of our approach and demonstrate computational efficiency and asymptotic optimality of the planner. 
\end{abstract}

\section{INTRODUCTION}
    \label{sec:intro}


As robots become more capable, they also become more adept at autonomously exploring remote environments, especially those that are hostile to humans.  Examples include deep ocean, planetary, and subterranean exploration. Because of inaccessibility and limited resources, these robots must be able to operate \emph{efficiently} and \emph{safely}.  However, these are often \emph{competing} objectives, i.e., safer performance requires more resources \cite{Lahijanian:RAL:2018}. This problem is exacerbated when robots are part of a network: frequent information communication increases the probability of safely completing the mission, but can also lead to unacceptable resource consumption. This work focuses on this problem and aims to develop a framework for robot motion and communication planning that guarantees safety while minimizing resource cost. 

Consider the case of a robot in a lunar exploration scenario with a network of remote sensors. The robot is tasked to navigate to a scientifically interesting location and avoid craters and boulders. The robot relies on the sensor network for localization, but the remote sensors have limited battery life, and transmitting information is costly. 
The network must also limit communication due to bandwidth constraints. Hence, the robot should find a motion and communication plan that minimizes the resource cost on the sensors while guaranteeing safety.

A ubiquitous technique to manage communication cost is event-triggered (ET) estimation \cite{Trimpe2015, Ouimet2018, Shi2014, Wu2013}, where communication only occurs when the information is deemed useful. This typically takes the form of a threshold on some useful quantity, e.g., the Kalman Filter (KF) innovation \cite{Ouimet2018, Shi2014, Wu2013}. ET allows the user to lower resource consumption by changing this threshold, thus reducing the volume of transmitted information and trading off estimation accuracy with resource use. Work \cite{Theurkauf2022} introduces a method that guarantees optimal trade-offs by accounting for both task performance and resources. While providing a comprehensive Pareto Front and optimal ET thresholds, that method does not generate a motion plan; it also requires heavy computation power, resulting in lengthy computation times and precluding onboard implementation.

Sampling-based motion planners are well established tools that can rapidly find solutions to complex problems \cite{Elbanhawi2014,LaValle1999,Luders2010,Blackmore2011,Li2016SST}. The first developed planners deal with deterministic dynamics and operate solely in the state space \cite{LaValle1999, KavrakiPRM1996}. Later formulations extend those algorithms to obtain asymptotically-optimal planners such as SST and SST* \cite{Li2016SST}, and incorporate motion uncertainty via chance constraints on the probability of collision (e.g., CC-RRT), \cite{Luders2010,Blackmore2011, pairet2021}. These techniques have recently been extended to accommodate measurement uncertainty with Gaussian belief trees \cite{Ho2022, Bry2011}, and feedback-based information roadmaps \cite{firm2014}. However, while these fast and efficient methods have been extended to accommodate other uncertainties,
no sampling-based motion planning algorithm currently exists that incorporates ET estimation.

In this paper, we develop a sampling-based algorithm for \textit{motion and ET-communication} (METC) planning that, for a given system and environment, generates both a motion and ET estimation threshold plan. This algorithm is fast and efficient, and provides guarantees on the point-wise probability of collision and reaching the goal (i.e., safety constraints), while minimizing communication cost. We first derive an offline method to predict the expected state distribution, and provide the corresponding bounds for this distribution. Using these results, we then derive a sampling-based algorithm that generates the METC plans. We show the efficacy of our methods in several case studies. The results show that we can successfully generate plans for a variety of environments, these plans are valid with respect to the safety constraint, and optimality increases with increased computation time. We additionally compare the results of our algorithm to a Pareto optimal strategy in \cite{Theurkauf2022}, showing that we can generate plans with similar performance with orders of magnitude smaller computation time.

In summary, our contributions are: (i) a derivation of propagation equations for the distribution across states under ET estimation, (ii) a derivation of the bounds for these distributions, (iii) an algorithm for generating METC plans 
that minimize communication costs while respecting a safety constraint, and (iv) 
a series of case studies and benchmarks that demonstrate the algorithm's efficiency and optimality.

\section{PROBLEM FORMULATION}
    \label{sec:problem}
We consider a scenario where a lunar robot must navigate to a scientifically interesting goal region while avoiding obstacles. The robot receives measurements from resource and bandwidth constrained remote sensors. The goal is to generate METC plans that respect a safety constraint on obstacle collision and reaching the goal region, while simultaneously minimizing communication cost. 
Below, we formalize this problem.

\subsection{Robot Motion and Remote Sensor Models}
The motion of the robot is uncertain and described by 
\begin{equation}
\label{eq:sysDyn}
    x_{k+1} = A x_k + B u_k + w_k, \quad w_k \sim \mathcal{N}(0,Q),
\end{equation}
where $x_k\in \mathcal{X} \subseteq \mathbb{R}^n$ is the state, $u_k\in \mathcal{U}\subseteq\mathbb{R}^p$ is the control, $A\in\mathbb{R}^{n\times n}$ and $B\in\mathbb{R}^{n\times p}$, and $w_k \in \mathbb{R}^n$ is a random variable that represents a zero-mean Gaussian distributed noise with covariance $Q\in\mathbb{R}^{n\times n}$.

The robot receives measurements from a remote sensor network (e.g. beacons that provide one-way ranging measurements). We assume that the robot communicates with the closest sensor, and therefore only fuses one measurement at a time. For ease of presentation, we assume that the remote sensors are identical, but we emphasize that extending to different models is trivial. Hence, the sensor network can be represented as a single sensor with model
\begin{equation}
\label{eq:sysMeas}
    y_k = Cx_k + v_k,\quad v_k \sim \mathcal{N}(0,R),
\end{equation}
where $y_k\in \mathbb{R}^m$ is the measurement, $C\in\mathbb{R}^{m\times n}$, and $v_k \in \mathbb{R}^m$ is a random variable that represents zero-mean Gaussian distributed sensor noise with covariance $R\in\mathbb{R}^{m\times m}$.

The initial state of the robot is described by a Gaussian distribution $x_0 \sim \mathcal{N}(\hat{x}_0,\Sigma_0)$, with mean $\hat{x}_0 \in \mathbb{R}^n$ and covariance $\Sigma_0 \in \mathbb{R}^{n \times n}$. We assume the robot is fully controllable and observable, and that the covariance matrices $Q$ and $R$ are positive definite.

\subsection{Event-triggered Estimation}
Due to process and measurement noise, the robot's true state is unknown. Hence, an estimator maintains a probability distribution over the states, $x_k \sim b(x_k)$, called the \emph{belief}. Note that this is an online estimate conditioned on observed measurements. To conserve communication costs, the system operates with the KF innovation-based ET estimator presented in \cite{Wu2013}, where measurements are only communicated when they are ``surprising''. That work derives a recursive minimum mean square error (MMSE) estimator under the assumption that the belief is Gaussian (see Sec.~\ref{sec:ETfilter} for details). For a given a threshold $\delta_k \in \Delta = \mathbb{R}_{> 0}$, measurement $y_k$ is \textit{surprising} if the norm of the Mahalanobis (whitening) transformation of $z_k = y_k - C \hat{x}_k$ is larger than threshold $\delta_k$.


We define $\gamma_k \in \{0,1\}$ to be the triggering indicator, i.e., $\gamma_k=1$ if the measurement is sent, and $\gamma_k=0$ otherwise. 
When $\gamma_k = 0$, the robot is \emph{implicitly} informed that $y_k$ is not surprising; the key advantage of ET filter is the exploitation of this information to improve state estimation.  
In this framework, the threshold $\delta_k$ is a design parameter, which trades off estimation accuracy with resource cost. We seek a method for optimally setting this threshold.

\subsection{Controller}
The robot is equipped with a trajectory following controller. Given a nominal trajectory as a sequence of nominal control inputs $\check{U}_0^T=(\check{u}_0, \ldots, \check{u}_T)$ and nominal states $\check{X}_0^T=(\check{x}_0, \ldots, \check{x}_T)$, the feedback controller is: $u_k = \check{u}_{k-1} - K(\hat{x}_k - \check{x}_k)$, where $K$ is the controller gain. Under this controller, the closed-loop system dynamics become: 
$$x_k = Ax_{k-1}+B(\check{u}_{k-1} - K(\hat{x}_k - \check{x}_k)) + w_k.$$ 
The goal of this work is to compute the nominal trajectory $(\check{U}_0^T, \check{X}_0^T)$ along with the sequence of ET estimation thresholds $\Delta_0^T=(\delta_0, ..., \delta_T)$ that satisfy the mission objectives and constraints described below.
We define a \textit{motion and ET-communication (METC) plan} to be $(\check{U}_0^T, \check{X}_0^T, \Delta_0^T )$.

\subsection{Mission Objectives}
The mission consists of three objectives: respect constraints on the probability of reaching the goal and avoiding obstacles, and minimize resource consumption. Here, we formalize these objectives.

The environment contains a set of obstacles, $\mathcal{X}_{obs} \subset \mathbb{R}^n$, and a goal region, $\mathcal{X}_{goal} \subset \mathbb{R}^n$. The probability of collision at time step $k$ is defined as:
\begin{equation}
\label{eq:Pcoll}
    P(x_k \in \mathcal{X}_{obs}) = \int_{\mathcal{X}_{obs}} b(x_k)(s) ds,
\end{equation}
where $b(x_k)(s)$ is the distribution $b(x_k)$ evaluated at state $s\in\mathcal{X}$. The probability of terminating in the goal region is:
\begin{equation}
\label{eq:Pgoal}
    P(x_T \in \mathcal{X}_{goal}) = \int_{\mathcal{X}_{goal}} b(x_T)(s)ds,
\end{equation}
where $x_T$ is the terminal point on a trajectory.

The 
third objective is to conserve resources for the remote sensors. 
The cost of transmitting a single measurement at time step $k$ is $c_m \gamma_k$, where $c_m >0$ is a constant. 
Since trigger $\gamma_k$ depends on the triggering threshold $\delta_k$, the \emph{total communication cost} for a trajectory with $T$ time steps and $\Delta_0^T = (\delta_0,\ldots,\delta_T)$ is $\mathcal{J}_T(\Delta_0^T) = \sum_{k=1}^{T} c_{m}\gamma_k$. 
Note that $\gamma_k$ is a random variable; hence, 
the total cost is considered in expectation: 
\begin{equation}
\label{eq:costfunction}
    J_T(\Delta_0^T) = \mathbb{E}[\mathcal{J}_T(\Delta_0^T)]=\sum_{k=1}^T c_{m}\mathbb{E}[\gamma_k] = \sum_{k=1}^T c_{m}\Gamma(\delta_k)
\end{equation}
where $\Gamma(\delta_k)$ is the expected triggering rate for $\delta_k$.
Because we focus on sensor resources, \eqref{eq:costfunction} contains only communication cost, but we emphasize that it can be easily extended to include terms for path length and control effort. 

\subsection{Problem Statement}
\label{sec:problemstatement}
Given a robot with dynamics in \eqref{eq:sysDyn}, sensor network with measurement model in \eqref{eq:sysMeas}, set of obstacles $\mathcal{X}_{obs}$, goal region $\mathcal{X}_{goal}$, and safety probability bound $p_{\safe} \in [0,1]$, compute 
an optimal METC plan $(\check{U}_0^T, \check{X}_0^T, \Delta_0^T)^*$ that minimizes the expected total communication cost, i.e., 
\begin{equation}
    \label{eq:cost minimization problem}
    (\check{U}_0^T,\check{X}_0^T,\Delta_0^T)^* = \arg \min_{\check{u}_k,\check{x}_k, \delta_k, T}  J_T(\Delta_0^T)
\end{equation}
subject to the following constraints on the point-wise probability of collision and reaching the goal:
\begin{align}
    &P(x_k \in \mathcal{X}_{obs}) < 1-p_{\safe}, \quad \forall k \in [0,T] \label{eq:cc-obs} \\
    &P(x_T \in \mathcal{X}_{goal}) > p_{\safe}. \label{eq:cc-goal}
\end{align}

The key challenge is accounting for the uncertainty introduced by ET estimation, which requires forecasting state distributions over both unknown measurements and unknown triggers. Measurement uncertainty in a KF is accounted for in \cite{Bry2011}; however, there are no existing methods to forecast uncertainty over a triggering condition. Straightforward attempts to extending \cite{Bry2011} to ET quickly run into problems of inter-dependency between variables and exponentially exploding belief trees. We address these challenges by developing a method to propagate bounds on the state distributions under a given choice of $\delta_k$ and using these to check the safety constraint. We then develop a planning algorithm by integrating these methods with the sampling-based algorithm in \cite{Ho2022} to generate METC plans that asymptotically minimize communication costs while respecting safety constraints.

\section{PRELIMINARIES}
    \label{sec:preliminaries}
    
\subsection{Belief Prediction under Kalman Filter}
\label{sec:KFpropagation}

For a linearizable and controllable system operating under a standard Kalman Filter, \cite{Bry2011} presents a method to forecast the belief over the state while accounting for the fact that the measurements are unknown random variables \emph{a priori}. This \emph{expected belief}, $\textbf{b}(x_k)$, is defined with respect to all possible measurements as:
\begin{align}
    \label{eq: expected belief}
    \textbf{b}(x_k) = \mathbb{E}_Y[b(x_k \mid x_0, y_{0:k})]
    = \int_{y_{0:k}} \hspace{-3mm} b(x_k \mid x_0, y_{0:k}) pr(y_{0:k})dy \nonumber
\end{align}
This forecast enables the evaluation of the chance constraints.

For a given nominal trajectory, $\check{X}_0^T=(\check{x}_0, \ldots, \check{x}_T)$,  the expected belief $\textbf{b}(x_k)=\mathcal{N}(\check{x}_k, \Sigma_k+\Lambda_k)$, can be recursively calculated from an initial belief $b(x_0)$ using the belief propagation method from \cite{Bry2011}:
\begin{align}
    \Sigma_k^- &= A\Sigma_{k-1}A^T + Q, \quad \Sigma_k = \Sigma_k^- - L_k C\Sigma_k^-, \\
    \Lambda_k &= (A-BK)\Lambda_{k-1}(A-BK)^T + L_k C \Sigma_k^-,
\end{align}
where $\Sigma_k$ is the online uncertainty given by the KF, and $\Lambda_k$ is covariance of the forecasted state estimates $\hat{x}_k$ (note that $\hat{x}_k$ is a random variable offline). Intuitively, this distribution can be thought of as the sum of the online estimation error and the forecasted uncertainty from not-yet-known measurements that the system receives during execution.  In this work, we develop a method of forecasting belief under an ET Filter.

\subsection{MMSE Filter for Event-triggered Estimation}
\label{sec:ETfilter}

For state estimation,
we use the triggering scheme described in \cite{Wu2013}, which is based on the KF innovation. Recall that the trigger, $\gamma_k$, depends on the triggering threshold $\delta_k$, and that $\gamma_k = 1$ indicates that measurement $y_k$ is sent, and $\gamma_k=0$ otherwise. According to the MMSE estimator in \cite{Wu2013}, the estimate of the state is Gaussian with \textit{a priori} update:
\begin{equation}
\label{eq:ETaprioriUpdate}
    \hat{x}_k^- = A \hat{x}_{k-1} + B u_k, \quad \Sigma_k^- = A \Sigma_{k-1} A^T + Q,
\end{equation}
and \textit{a posteriori} update given by:
\begin{align}
\label{eq:ETaposterioriUpdate}
    &\hat{x}_k = \hat{x}_k^- + \gamma_k L_k z_k, \quad z_k = y - C\hat{x}_k^- \\
    &\Sigma_k =  \Sigma_k^- - \Big[\gamma_k + (1-\gamma_k) \beta(\delta_k) \Big] L_k C\Sigma_k^-.
\end{align}
where 
$L_k=\Sigma_k^- C^T[C \Sigma_k^-C^T + R]^{-1}$ is the KF gain. The term ${\beta(\delta_k)}$ is a scalar multiplier that effectively attenuates the KF gain in the covariance update as a function of $\delta_k$, and is given by
\begin{equation}
\beta(\delta) = \frac{2}{\sqrt{2\pi}}\delta e^{-\frac{\delta^2}{2}}\left[ 1 - 2\mathcal{Q}(\delta) \right]^{-1},
\end{equation}
where
\begin{equation}
    \mathcal{Q}(\delta) \triangleq \int_\delta^{+\infty} \frac{1}{\sqrt{2\pi}} e^{-\frac{x^2}{2}}dx.
\end{equation}

The triggering condition takes the form: 
\begin{equation}
    \gamma_k = 
    \begin{cases}
    0 & \text{if $\|\epsilon_k\|_\infty\leq\delta_k$}\\
    1 & \text{otherwise},
    \end{cases}
\end{equation}
where 
$\epsilon_k$ is the Mahalanobis (whitening) transformation of the KF innovation, $z_k$.
Note that, because of the whitening transformation, $\epsilon_k$ is always distributed as a standard normal.  Hence, the expected value of $\gamma_k$
is solely dependent on $\delta_k$ and is given by $\Gamma(\delta_k)= 1-\left[ 1 - 2\mathcal{Q}(\delta_k)\right]^m$, where $m$ is the dimension of the measurement vector.

\section{BELIEF PREDICTION UNDER ET FILTER}
\label{sec:ET Bounds}
In this section, we present a novel method to predict the state distribution offline under an ET filter. The resulting distribution can be used to check the validity of the safety constraints.  Based on this method, we devise the planning algorithm in Sec. \ref{sec:ET-GBT}.

A key requirement for the method of belief prediction in Sec.~\ref{sec:KFpropagation} is that the distribution over state estimates, $\hat{x}_k$, is Gaussian, and defined by $\hat{x}_k \sim \mathcal{N}(\check{x}_k, \Lambda_k)$. When this distribution is known, it can be used to represent the joint distribution over $x_k$ and $\hat{x}_k$, which can in turn be marginalized to obtain the expected distribution over states, $x_k \sim \textbf{b}(x_k)=\mathcal{N}(\check{x}_k,\Sigma_k+\Lambda_k)$. This marginal is then used to evaluate the chance constraints in \eqref{eq:cc-obs} and \eqref{eq:cc-goal} offline.

Unfortunately, under ET estimation the covariance is updated under two randomly switching modes (according to $\gamma_k=0$ or $\gamma_k=1$), which means that the methods described in Sec. III.A cannot be directly applied to ET estimation. If the trigger $\gamma_k$ is treated as an unknown random variable, then the expected belief must be taken with respect to all possible $\gamma_k$, and is not Gaussian. However, if $\gamma_k$ is assumed to be given, then the distribution over $\hat{x}_k$ is Gaussian. This can be seen by examining the ET estimation update in \eqref{eq:ETaposterioriUpdate}: when $\gamma_k$ is given, the conditional dependency of $\hat{x}_k$ on $\gamma_k$ vanishes and the Gaussian property is preserved during the update. The mean and covariance can be obtained by evaluating $\mathbb{E}[\hat{x}_k]=\check{x}_k$ and $\mathbb{E}[(\hat{x}_k-\check{x}_k)(\hat{x}_k-\check{x}_k)^T]=\Lambda_k$ respectively. The resulting distribution $x_k \sim \mathcal{N}(\check{x}_k,\Sigma_k+\Lambda_k)$ can then be calculated, with $\Sigma_k$ given by the ET update in \eqref{eq:ETaposterioriUpdate} and $\Lambda_k$ given by: 
\begin{align}
\label{eq:ETcovupdate}
    \Lambda_k &= (A-BK)\Lambda_{k-1}(A-BK)^T + \gamma_k L_k C \Sigma_k^-.
\end{align}
The assumption that $\gamma_k$ is given is significant. Instead of taking the expected belief with respect to all possible $\gamma_k$, we assume a realization of some specific sequence of $\gamma_k$. However, if we are only concerned with bounding the probability of being within some region, this assumption can be easily accommodated. We simply need to determine the triggering condition that results in the expected belief with the highest, or lowest, probability of being in that region. 

While the expectation of triggering at any given time step is  easy to calculate, the corresponding expected belief is not. However, because we are considering a chance constraint, it is sufficient to only consider an upper bounding belief, which  subsumes all covariances produced by any triggering condition. This bound can be used for collision checking as described in Section \ref{sec:validityChecking} to guarantee satisfaction of the chance constraints.
%
We use a bounding method similar to the one described in \cite{Li2018}. 

Let the scalar constants $\underline{a}, \bar{a}, \underline{k}, \bar{k}, \underline{c}, \bar{c},\underline{q}, \bar{q}$, $\underline{r}, \bar{r} \in \mathbb{R}_{>0}$ define bounds such that
\begin{equation}
\label{eq:SysBoundAssumptions}
\begin{split}
    &\underline{a}^2 I_n \leq A A^T \leq \bar{a}^2 I_n, \qquad \underline{c}^2 I_m \leq C C^T \leq \bar{c}^2 I_m, \\
    &\underline{k}^2 I_n \leq (A -BK)(A-BK)^T \leq \bar{k}^2, \\
    &\underline{q} I_n \leq Q \leq \bar{q} I_n,  \qquad \qquad\; \underline{r}I_m \leq R \leq \bar{r} I_m,
\end{split}
\end{equation}
where $I_m$ is the $m\times m$ identity matrix, and inequality relation $\textit{X} \geq \textit{Y}$ for square matrices $\textit{X}$ and $\textit{Y}$ implies $\textit{X}-\textit{Y}$ is positive semi-definite. 
Note that the existence of these positive bounds requires that the eigenvalues of $A$, $(A-BK)$, and $CC^T$ be real.  Based on these parameters, we derive an upper bound for the covariance of the belief in the theorem below.

\begin{theorem}
    \label{theorem:covariance-bound}
    Consider the predicted belief for ET estimation given by $\textbf{b}(x_k)=\mathcal{N}(\check{x}_k,\Sigma_k+\Lambda_k)$, where the covariance $\Sigma_k+\Lambda_k$ is recursively updated according to \eqref{eq:ETcovupdate}. Then, this covariance can be recursively bounded by
    \begin{equation}
        \label{eq:upperboundcovariance}
        \Lambda_k + \Sigma_k \leq (\bar{\lambda}_k + \bar{p}_k) I_n
    \end{equation}
    where $\bar{\lambda}_k,\bar{p}_k \in \mathbb{R}_{\geq 0}$ are given by
    \begin{align}
        &\bar{\lambda}_k = \bar{k}^2\bar{\lambda}_{k-1} + \frac{\bar{c}^2\left( \bar{p}_{k-1}\bar{a}^2 + \bar{q} \right)^2}{\underline{c}^2\left( \underline{p}_{k-1}\underline{a}^2 + \underline{q} \right) + \underline{r}} \nonumber\\
        &\bar{p}_k = \left( \frac{1}{\bar{p}_{k-1}\bar{a}^2 + \bar{q}} + \frac{\beta\underline{c}^2}{\bar{r} + \left( 1-\beta \right)\bar{c}^2\left( \bar{a}^2\bar{p}_{k-1}+\bar{q}\right) } \right)^{-1}
        \nonumber
    \end{align}
with $\underline{p}_k= \left( 1/\underbar{q} + \bar{c}^2/\underbar{r} \right)^{-1}$.
The bounds are initialized as:
\begin{align*}
    &\bar{\lambda}_0 = \max \left(Eigenval\left(\Lambda_0\right)\right), \\
    &\bar{p}_0 = \max\left(Eigenval\left(\Sigma_0\right)\right), \; \underline{p}_0 = \min\left(Eigenval\left(\Sigma_0\right)\right).
\end{align*}

\end{theorem}

\begin{proof}
We begin by deriving the individual bound for $\Lambda_k$ such that $\Lambda_k \leq \bar{\lambda}_k I$. Consider the case for $\gamma_k=0$, the update equation can be written as:
\begin{equation}
\label{eq:lambdaUpdateGamma0}
    \Lambda_k = (A-BK)\Lambda_{k-1}(A-BK)^T.
\end{equation}
This can be simply bounded as $(A-BK)\Lambda_{k-1}(A-BK)^T\leq\bar{k}^2\bar{\lambda}_{k-1}$. Now consider the case $\gamma_k=1$. The update equation can be written as:
\begin{multline}
\label{eq:lambdaUpdateGamma1}
    \Lambda_k = (A-BK)\Lambda_{k-1}(A-BK)^T \\
    + \Sigma_k^- C^T(C\Sigma_k^- C^T + R)^{-1} C \Sigma_k^-.
\end{multline}
The term $(A-BK)\Lambda_{k-1}(A-BK)^T$ is the same as \eqref{eq:lambdaUpdateGamma0}, and can be bounded in the same way. In order to bound the second term, we must first  derive bounds on $\Sigma_k^-$, the covariance from the a priori filter update. This can be bounded in terms of the previously defined bounds on the a posteriori update, $\underline{p}_k$ and $\bar{p}_k$. From \eqref{eq:ETcovupdate} we can generate the bounds:
\begin{equation}
\label{eq:proofEq0}
    \left( \underline{a}^2\underline{p}_{k-1} + \underline{q} \right) I \leq \Sigma_k^- \leq \left(\bar{a}^2\bar{p}_{k-1} + \bar{q}\right) I
\end{equation}
Next, examine the term $(C\Sigma_k^- C^T + R)^{-1}$. We require the following lemma to calculate the bound:
\begin{lemma}
\label{lem:lemma2}
    Let $\text{X}\in\mathbb{R}^{n\times n}$ be bounded such that $\text{X}\geq \underline{\mathrm x}I$. Then, $\text{X}^{-1}\leq 1/\underline{\mathrm x} I$
\end{lemma}
Applying Lemma \ref{lem:lemma2} we obtain: 
\begin{equation}
\label{eq:proofEq1}
(C\Sigma_k^- C^T + R)^{-1} \leq (\underline{c}^2(\underline{p}_{k-1}\underline{a}^2 + \underline{q}) + \underline{r})^{-1}I    
\end{equation}
Using \eqref{eq:proofEq1} and \eqref{eq:proofEq0}, we can write the full bound in \eqref{eq:upperboundlambda}. It is simple to see that this bound is larger than the bound for the case $\gamma_k=0$, and therefore is the true upper bound for all triggering conditions.

We use similar reasoning to gererate bounds on bounds on $\Sigma_k$. For the case $\gamma_k=1$, \eqref{eq:ETaposterioriUpdate} becomes the standard KF equations. Therefore, the upper and lower bounds for the case $\gamma_k=1$ are the same KF bounds derived in \cite{Li2016}: 
\begin{equation}
\left( \frac{1}{\underbar{q}} + \frac{\bar{c}^2}{\underbar{r}} \right)^{-1} I \leq \Sigma_{k,\gamma_k=1} \leq \left(\bar{p}_{k-1}\bar{a}^2 + \bar{q}\right)I
\end{equation}

We derive the bounds for the case $\gamma_k=0$ based on the inverse form of the covariance update equation \eqref{eq:ETaposterioriUpdate} for $\gamma_k=0$:
\begin{equation}
\begin{split}
    &\Sigma_{k,\gamma_k=0} = \\
    &\left( \left(\Sigma_k^-\right)^{-1} + \beta C_k^T \left(  R_k + \left(1-\beta\right)C_k\Sigma_k^- C_k^T \right)^{-1} C_k \right)^{-1}
    \nonumber
\end{split}
\end{equation}
Under straightforward manipulation and application of Lemma \ref{lem:lemma2}, this yields the bounds:
\begin{align}
    &\left( \frac{1}{\underbar{q}} + \frac{\beta\bar{c}^2}{\underbar{r} + (1-\beta)\underline{c}^2\underbar{q}} \right)^{-1} I \leq \Sigma_{k,\gamma_k=0} \\
    & \leq \left( \frac{1}{\bar{p}_{k-1}\bar{a}^2 + \bar{q}} + \frac{\beta\underline{c}^2}{\bar{r} + \left( 1-\beta \right)\bar{c}^2\left( \bar{a}^2\bar{p}_{k-1}+\bar{q}\right) } \right)^{-1} I
\end{align}
Note that the lowest lower bound on $\Sigma_k$ corresponds to the case $\gamma_k=1$. Similarly, the largest upper corresponds to $\gamma_k=0$. These widest bounds are presented in \eqref{eq:upperboundcovariance}, and are guaranteed to bound the ET filter covariance for any triggering condition.

Finally, we can bound the sum of the covariances $\Sigma_k + \Lambda_k$ by the sum of their respective upper bounds: $\Sigma_k + \Lambda_k \leq (\bar{p}_k + \bar{\lambda}_k) I$. Because these represent the largest upper bound and lowest lower bound for any triggering condition at any time step, the recursively calculated sequence of bounds is guaranteed to bound the expected belief for any possible sequence of triggers.
\end{proof}

\section{ET-GBT PLANNING ALGORITHM}
    \label{sec:ET-GBT}
    This section introduces the Event-Triggered Gaussian Belief Trees (ET-GBT) algorithm, an adaptation of the Gaussian Belief Trees (GBT) algorithm in \cite{Ho2022} for ET estimation in order to minimize communication cost. 

\subsection{Gaussian Belief Trees}
We first present a brief overview of the GBT motion planner from \cite{Ho2022}. There, a framework is developed for extending any kinodynamic tree-based motion planner to the belief space, where the edges are still nominal controllers and trajectories, but the nodes are Gaussian beliefs. The algorithm proceeds as follows. First a belief is randomly sampled.  Next, its closest node is computed using the 2-Wasserstein distance and extended by a random control input. The uncertainty covariance is propagated using the technique discussed in Sec.~\ref{sec:KFpropagation}. A new node is only added to the tree if it satisfies the chance constraints of probability of collision with obstacles, which is over-approximated using \cite{Blackmore2011, pairet2021, Park2017}. The process repeats until a solution is found.




\subsection{ET-GBT Algorithm}

ET-GBT adapts GBT in two fundamental ways: tree expansion and chance constraint validity checking. ET-GBT can be used to identify a valid tree that optimizes for communication cost in \eqref{eq:costfunction}, and satisfies the safety constraints \eqref{eq:cc-obs} and \eqref{eq:cc-goal}.  We optimize for this cost function using SST \cite{Li2016SST}, an asymptotically near-optimal planner. Alg.~\ref{alg:ETGBT} presents the pseudocode for our proposed algorithm.

\subsubsection{Tree Expansion}
\label{sec:tree expansion}
Instead of maintaining and propagating Gaussian beliefs, we propagate the bounds \eqref{eq:upperboundcovariance} on the beliefs under ET per the equations derived in Sec.~\ref{sec:ET Bounds}. 
The rest of the tree expansion algorithm follows intuitively from this main representation change. The SampleBelief() function is unchanged from GBT and operates analogously to the state sampler in an RRT search. The SelectNode() function is modified to select the `closest' belief node for extension using the 2-Wasserstein distance metric to the upper bounding belief. SampleDelta() has been added so that each edge of the tree corresponds to a triggering threshold as well as a nominal control input.

\begin{algorithm}[t]
\caption{ET-GBT}
\label{alg:ETGBT}
 \SetKwInOut{Input}{Input}\SetKwInOut{Output}{Output}
        \Input{$X$, $U$, $\mathcal{X}_{obs}$, $\mathcal{X}_{goal}$, $N$}
        \Output{Tree $G=(\mathbb{V}, \mathbb{E})$}
 $G \gets (\mathbb{V} \gets \{b_{init}\}, \mathbb{E} \gets \emptyset)$\\
\For{$N$ iterations}{
    $b_{rand} \gets$ SampleBelief()\\
    $\delta_{rand} \gets$ SampleDelta()\\
    $u_{rand} \gets$ SampleControl()\\
    $n_{select} \gets$ SelectNode()\\
    $n_{new} \gets \text{Extend}(n_{select},\delta_{rand})$\\
    \If{ValidPathCheck$(n_{select}, n_{new}, \delta_{rand})$}{
        $\mathbb{V} \gets \mathbb{V} \cup \{n_{new}\}$
        $\mathbb{E} \gets \mathbb{E} \cup \{edge(n_{select}, n_{new}, \delta)\}$
    }
    Prune($\mathbb{V}, \mathbb{E}$)
}
\Return{$G=(\mathbb{V}, \mathbb{E})$}
\end{algorithm}

\subsubsection{Chance Constraint Validity Checking}
\label{sec:validityChecking}
We use an over-approximation to check that the probability of collision is below the safety constraint \eqref{eq:cc-obs}. This allows for very fast constraint checking and preserves the efficiency of the sampling-based algorithm. First, we define the $p_{\safe}$ probability contour as the level set $\mathcal{L}_c=\{s \mid \textbf{b}(x_k)(s)=c\}$ with $c\in\mathbb{R}_{\geq 0}$. This contour is calculated such that the interior of $\mathcal{L}_c$ defines a region (volume), $\mathcal{A} \subset \mathcal{X}$, that contains $p_{\safe}$ probability mass
\begin{equation}
    \int_\mathcal{A} \textbf{b}(x_k)(s)ds = p_{\safe}.
\end{equation}
For a Gaussian distribution, level set $\mathcal{L}_c$ and region $\mathcal{A}$ are ellipsoidal and defined by the eigenvalues and eigenvectors of the covaiance matrix. 
If $\mathcal{A}$ and $\mathcal{X}_{obs}$ are non-intersecting, then $P(x_k \in \mathcal{X}_{obs})$ must be less than $ 1-p_{\safe}$, i.e.,
\begin{equation*}
    \mathcal{A} \cap \mathcal{X}_{obs} = \emptyset \quad \Rightarrow \quad P(x_k \in \mathcal{X}_{obs}) \leq 1-p_{\safe}.
\end{equation*}

The key insight is that the covariance bounds presented in Theorem \ref{theorem:covariance-bound} can be used to calculate the bounding contour that contains all possible contours for $\textbf{b}(x_k)$.
Specifically, the covariance bound matrix is $(\bar{\lambda}_k + \bar{p}_k) I_n$, i.e., diagonal with equal entries; 
hence, the contour is an $n$-dimensional sphere (n-sphere) with radius $r_k=t_\alpha(\bar{\lambda}_k + \bar{p}_k)$, where $t_\alpha$ is computed from the quantile function $\phi^{-1}$ of the n-dimensional Gaussian distribution such that $t_\alpha = -\phi^{-1}(0.5 \, p_{\safe})$.

Therefore, evaluation of Constraint \eqref{eq:cc-obs} can be efficiently computed by checking $\mathcal{A} \cap \mathcal{X}_{obs} = \emptyset$, which corresponds to checking for intersections of the obstacles with the n-sphere.
Evaluation of \eqref{eq:cc-goal} can be done similarly by determining whether the $p_{\safe}$ contour is completely enclosed by the goal region, i.e., $\mathcal{A} \subseteq \mathcal{X}_{goal}$.


We note that this method of checking for Constraints \eqref{eq:cc-obs} and \eqref{eq:cc-goal} introduces two main sources of conservatism. The first is inherent to using the probability contour; 
while non-intersection of $\mathcal{A}$ and $\mathcal{X}_{obs}$ implies $P(x_k \in \mathcal{X}_{obs}) \leq 1 - p_{\safe}$, intersection does not imply that $P(x_k \in \mathcal{X}_{obs}) > 1 - p_{\safe}$.
The second arises from bounding the true elliptical covariances with a spherical bound. There may be cases that the spherical bound intersects with obstacles, but the true elliptical covariance does not. This is compounded by the assumptions that $\gamma_k=0$ or $\gamma_k=1$ for all $k$ when computing the bounds. This does not occur in reality, so a bound calculated on these extremes is consequentially conservative. 

\subsection{Correctness, Completeness, and Optimality}


In this section, we show that ET-GBT (i) is sound and probabilistically complete with respect to the conservative constraint-validity-checking method presented above, and (ii) satisfies the conditions for asymptotic (near-)optimality of kinodynamcis planners such as SST or SST* \cite{Li2016SST}.


\begin{lemma}[Correctness]
    \label{lemma:correctness}
    Let $G = (\mathbb{V}, \mathbb{E})$ be the tree obtained from ET-GBT for some iterations $N \in \mathbb{N}$. Consider any tree node $v \in \mathbb{V}$. This node is guaranteed to satisfy the chance constraints in \eqref{eq:cc-obs}. Further, if the algorithm returns a solution, the final tree node is guaranteed to satisfy \eqref{eq:cc-goal}.
\end{lemma}
\noindent
The proof of this lemma follows directly from our method of chance constraint checking using the (conservative) covariance bound from Theorem~\ref{theorem:covariance-bound}.

\begin{theorem}[Probabilistic Completeness]
    ET-GBT is probabilistically complete with respect to the conservative upper bound, i.e., if there exists a solution using the upper bound, ET-GBT will find it almost surely as iterations $N \rightarrow \infty$.
\end{theorem}
\noindent
The proof of this theorem follows directly from Lemma~\ref{lemma:correctness} and the probabilistic properties of GBT.



For asymptotic optimality,
we first show that the cost function in \eqref{eq:costfunction} satisfies the conditions for an admissible cost function for asymptotic near-optimal algorithms according to [21]. Specifically, it is straightforward to see that \eqref{eq:costfunction} satisfies additivity, monotinicity, and non-degeneracy. In the following, we prove that it is also Lipschitz continuous in $\delta$.

\begin{figure*}
    \centering
    \begin{subfigure}{0.28\textwidth}
    \includegraphics[width=\textwidth]{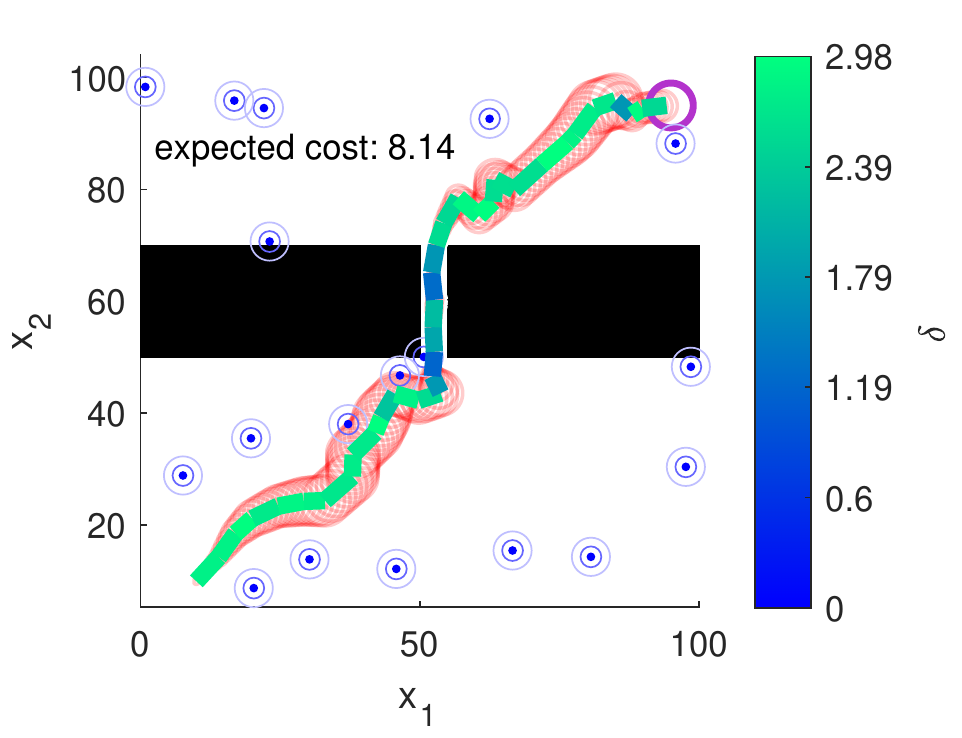}
    \vspace{-7mm}
    \caption{}
    \label{fig:Sys1Strategy}
    \end{subfigure}
    \hfill
    \begin{subfigure}{0.22\textwidth}
    \includegraphics[width=\textwidth]{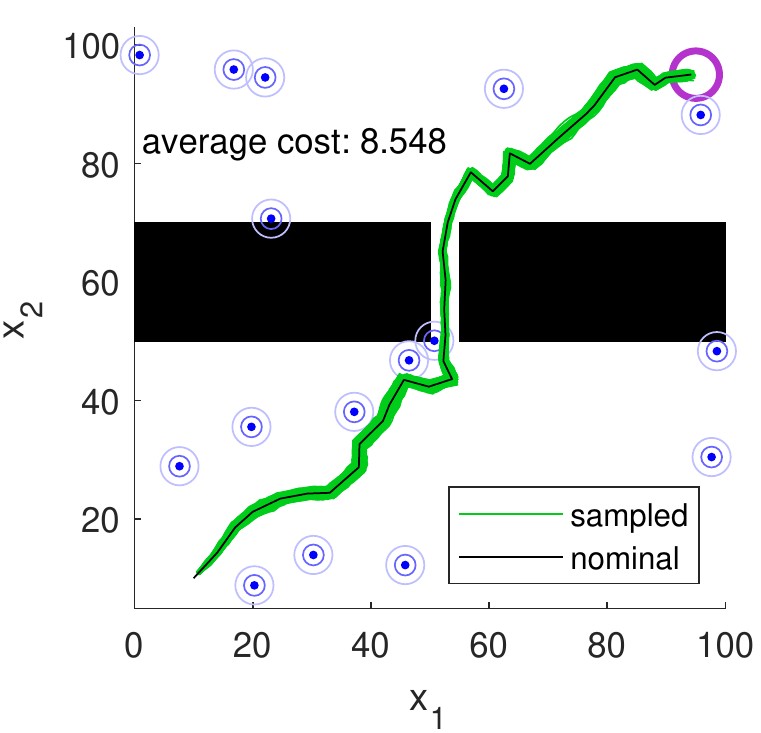}
    \vspace{-7mm}
    \caption{}
    \label{fig:Sys1MCRuns}
    \end{subfigure}
    \hfill
    \begin{subfigure}{0.2\textwidth}
    \includegraphics[width=\textwidth]{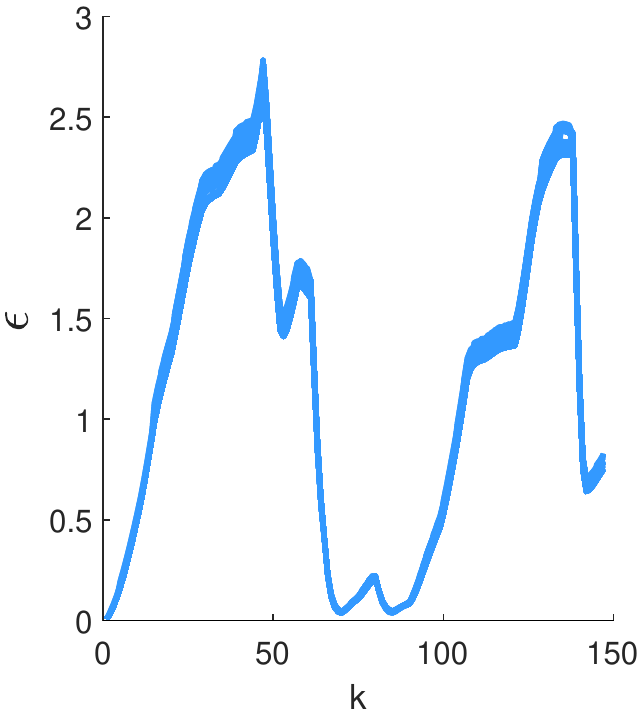}
    \vspace{-7mm}
    \caption{}
    \label{fig:Sys1covBounds}
    \end{subfigure}
    \hfill
    \begin{subfigure}{0.28\textwidth}
    \includegraphics[width=\textwidth]{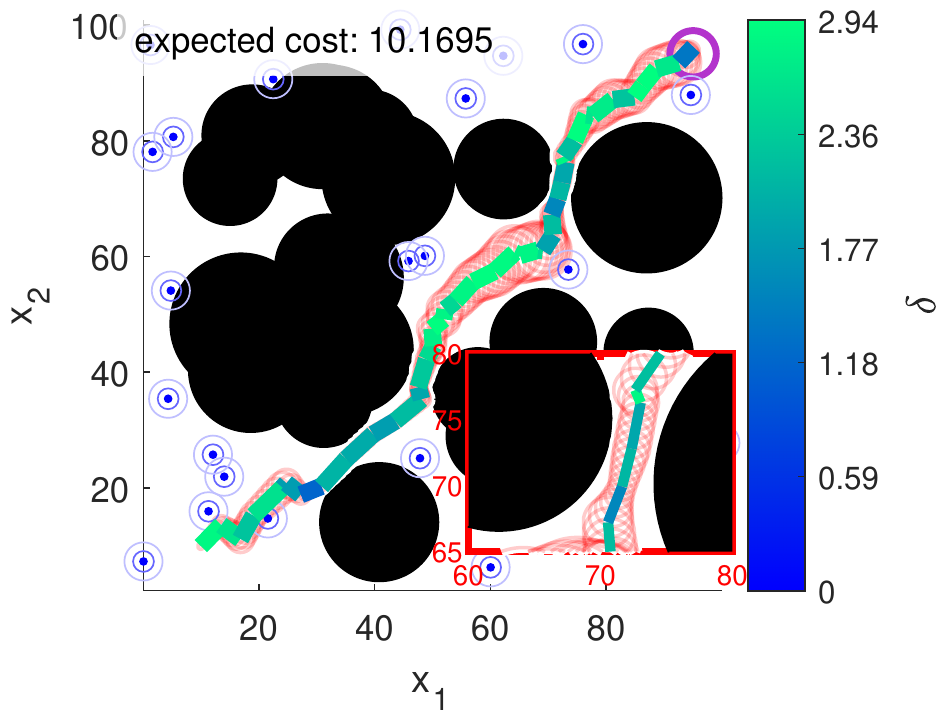}
    \vspace{-7mm}
    \caption{}
    \label{fig:scenario2_nominal}
    \end{subfigure}
\caption{(a) Nominal trajectory for 2D system, with $\check{X}_0^T$ plotted as lines colored by $\Delta_0^T$, and $99\%$ contours plotted as red circles.  The blue concentric circles represent the remote sensors. (b) $50$ sampled MC runs. (c) $50$ sampled MC runs of $\epsilon$. (d) Nominal trajectory for 2D system with random obstacles}
\vspace{-3.5mm}
\end{figure*}

\begin{lemma}[Lipschitz Continuity]
    The cost function in \eqref{eq:costfunction} is Lipschitz Continuous.
\end{lemma}
\begin{proof}
It is enough to show the Lipschitz continuity of one time step, since the cost function is a sum of stage costs at each time step \eqref{eq:costfunction}. For $\delta \in \positivereals$, we see that $\Gamma(\delta_t)=1-\left[ 1 - 2\mathcal{Q}(\delta_k)\right]^m,$ is everywhere differentiable, and also that $J_t$ has a bounded first derivative, since $\delta \in \positivereals$. Therefore, it is Lipschitz continuous in $\delta$.
\end{proof}

\noindent
Finally, we show that the belief bound dynamics is Lipschitz continuous w.r.t. the triggering threshold ($\delta$), controls, and the belief state. Continuity w.r.t. belief and controls follows directly from the Lipschitz continuity of the linear dynamics. For $\delta$, we show that the derivative of the belief bound w.r.t. $\delta$ is bounded (see \cite{Theurkauf2022extended} for full proof).
Under these conditions, the ET-GBT planner inherits the asymptotic optimality properties of the underlying sampling based planner (similar to GBTs 
in \cite{Ho2022}).

\section{EVALUATIONS}
    \label{sec:eval}

We implemented ET-GBT with the asymptotically near-optimal planner SST \cite{Li2016SST} in the Open Motion Planning Library (OMPL)~\cite{sucan2012the-open-motion-planning-library}.
We evaluated our algorithm on two systems and two environments, generating METC plans for each scenario. We ran Monte Carlo (MC) simulations to collect statistics on: resource cost, collision probability and probability of reaching the goal. Benchmarking shows improved optimality, i.e., lower cost with longer runtimes. We also provide a comparison with the method in \cite{Theurkauf2022}. 
To emphasize the cost savings for the remote sensors, we considered $c_{m}=1$ in cost function \eqref{eq:costfunction}.
All benchmarks were computed single-threaded on a $3.6$ GHz  CPU with 16 GB RAM. 

\emph{Simple 2D System:}
We study the robotic system from \cite{Bry2011}, with dynamics $x_{k+1} = x_k + u_k + w_k$ and measurements $y_k = x_k + v_k$. The noise is distributed as $w_k \sim \mathcal{N}(0,0.1^2I)$ and $v_k \sim \mathcal{N}(0,0.1^2I)$.

We first consider an environment with a narrow corridor, shown in Figure \ref{fig:Sys1Strategy}, and constraint $p_{\safe}=0.99$. 
We ran ET-GBT for $1$ minute, generating the METC plan shown in Figure \ref{fig:Sys1Strategy}. 
Note the motion plan consists of a nominal state trajectory for the robot and a triggering threshold $\delta$ plan for the sensors. 
In Figure \ref{fig:Sys1Strategy}, the sensors are shown as concentric circles, and the robot's state trajectory is colored by the corresponding triggering threshold for the sensors. 
We see the $\delta$ threshold relax in regions far from obstacles, indicating a lower trigger rate, causing the belief bounds to expand, and accruing less cost. As the robot traverses the narrow corridor, the threshold tightens, causing the belief bounds to contract in order to satisfy the safety constraint. The threshold decreases at the end to ensure the robot terminates within the goal region. 

The generated nominal plan was validated by running $3000$ MC simulations, none of which collided with obstacles. A sampling of 50 of these trajectories is shown in Figure \ref{fig:Sys1MCRuns}. This affirms that the chance constraint has been met, but also indicates the bounds are very conservative. 

We verified the derived upper bound, $\Lambda_k+\Sigma_k\leq(\bar{\lambda}_k+\bar{p}_k)I$ by checking that the matrix difference, $(\bar{\lambda}_k+\bar{p}_k)I-(\Lambda_k+\Sigma_k)$, is positive semi-definite. This condition is checked by ensuring the eigenvalues of the matrix difference, denoted by $\epsilon = Eigenval((\bar{\lambda}_k+\bar{p}_k)I-(\Lambda_k+\Sigma_k))$, are all positive, as shown in Figure \ref{fig:Sys1covBounds}. Note that while the bound is valid, it is very loose, contributing to conservativeness in collision checking. A future direction is to investigate a tighter bound and reduce the conservatism of our algorithm. 

\emph{Benchmarks:}
Next, we performed benchmark analysis to demonstrate that the optimality improves with increased computation time. In addition to the 2D system, we evaluated our algorithm on a second order unicycle system with dynamics $\dot{x} = v \cos{(\varphi)}, \dot{y} = v \sin{(\varphi)}, \dot{\varphi} = \omega, \dot{v} = a$, and feedback linearized according to \cite{unicycle}.

We consider two environments: the narrow corridor environment and randomize environments. Each instance of the randomized environment contains $15$ circular obstacles, with obstacle centers, $c_o$ distributed as $c_o \sim \mathcal{U}(0,100)$, and obstacle radius, $r_o$, distributed as $r_o\sim\mathcal{N}(10,2)$.
A sample environment and the generated plan for the 2D system are shown in Figure~\ref{fig:scenario2_nominal}. For both environments, we conducted $100$ trials with computation times from $10$ seconds to $200$ seconds. The results of this analysis are presented in Table~\ref{tab:benchmarkingLinear}. In each case, the average cost of the plans decreases with computation time, showing asymptotic optimal behavior.



\begin{table}[h]
    \centering
    \caption{Benchmarking results.}
    \label{tab:benchmarkingLinear}
    \scalebox{1.0}{
    \tabcolsep=0.11cm
    \begin{tabular}{l|c|c|c|c}
     & \multicolumn{2}{c|}{2D System Cost} & \multicolumn{2}{c}{Unicycle System Cost}\\
    Time (s)  & {Narrow} & {Random} & {Narrow} & {Random}
    \\\hline
     $10$ & $38.47 \pm 1.0$ & $10.7 \pm 1.2$ & $60.2 \pm 2.2$ &  $105.3 \pm 12.4$ \\
         $25$ & $19.76 \pm 0.8$ &  $7.53 \pm 0.9$ & $34.8 \pm 0.9$ & $54.4 \pm 8.7$ \\
         $50$ & $11.40 \pm 0.3$ &$6.61 \pm 0.91$ & $25.4 \pm 0.6$ &  $31.9 \pm 2.9$ \\
         $100$ & $8.47 \pm 0.2$ &$5.94 \pm 0.89$ & $20.5 \pm 0.4$ & $25.0 \pm 2.0$ \\
         $200$ & $8.13 \pm 0.1$& $5.54 \pm 0.87$ & $18.3 \pm 0.3$ & $21.5 \pm 1.6$
    \end{tabular}
    }
    \vspace{-2mm}
\end{table}

\emph{Comparison to Pareto-optimal Method in \cite{Theurkauf2022}:}
To showcase the efficiency of our approach, we compare our method agaist the one proposed in \cite{Theurkauf2022} by considering
the experiment for open trajectory scenario in \cite{Theurkauf2022}.  That method computes a Pareto point with probability of goal and collision $0.95$ and $0.0$ respectively, and cost $30.95$. The computation time is in the order of hours.
We ran ET-GBT on the same trajectory and discrete $\delta$ values 100 times. 
ET-GBT returns a plan with very similar average expected cost of $30.56$ within $0.5$ seconds.  The computation time for ET-GBT is significantly smaller than the method in \cite{Theurkauf2022}, which relies on MC sampling to build an MDP abstraction. We note that, since the trajectory is far from any obstacles for this scenario, the conservativeness of ET-BGT does not have a profound effect on the result. Nevertheless, this comparison validates the speed and effectiveness of our optimization approach.  
    
\section{CONCLUSION}
    \label{sec:conclusion}
    This paper considers the problem of generating METC plans that satisfy safety constraints while minimizing communication costs. We develop a novel method of propagating the expected belief under ET estimation, as well the corresponding covariance bounds. We use these techniques to develop a fast and efficient sampling-based METC planning algorithm. Case studies and benchmarking demonstrate the efficacy, speed, and asymptotic optimality of the algorithm.

These methods are limited by the conservativeness of the approximations, which could be addressed in future work by tightening the covariance bounds, or by developing a less conservative collision checking method.


\bibliographystyle{IEEEtran}
\bibliography{refs}

\end{document}